\documentclass{article}

\usepackage[final]{neurips_2019}

\usepackage{tikz}
\usepackage{pgfplots}
\definecolor{mycolor1}{rgb}{0.32157,0.32157,0.32157}%

\usepackage[utf8]{inputenc}     % allow utf-8 input
\usepackage[T1]{fontenc}        % use 8-bit T1 fonts
\usepackage{hyperref}           % hyperlinks
\usepackage{url}                % simple URL typesetting
\usepackage{booktabs}           % professional-quality tables
\usepackage{amsfonts}           % blackboard math symbols
\usepackage{nicefrac}           % compact symbols for 1/2, etc.
\usepackage{microtype}          % microtypography
\usepackage{natbib}
\usepackage{graphicx}           % includegraphiscs
\usepackage{amsmath}            % split
\usepackage{amsthm}  
\usepackage{bbm}% proof
\usepackage{url}
\usepackage[export]{adjustbox}

\newtheorem{theorem}{Theorem}
\newtheorem{lemma}{Lemma}
\DeclareMathOperator{\Tr}{Tr}
\newenvironment{hproof}{%
  \proof}{\endproof}
\newcommand*\samethanks[1][\value{footnote}]{\footnotemark[#1]}

\title{Learning Interpretable Disease\\Self-Representations for Drug Repositioning}

\author{%
  \begin{tabular}{ c c c }
  Fabrizio Frasca$^{1,2}\thanks{Equal contribution, authors are in alphabetical order.}$ & Diego Galeano$^{3}$\samethanks
  & Guadalupe Gonzalez$^{4}$\\[1mm]
  Ivan Laponogov$^{4}$
  & Kirill Veselkov$^{4}$
  & Alberto Paccanaro$^{3}$\\[1mm]
  ~~~ & Michael M. Bronstein$^{1,2,4}$ & ~~~
  \end{tabular}
  \vspace{2mm}\\
  \begin{tabular}{ c c c c }
   $^1$USI Lugano & $^2$Twitter & $^3$Royal Holloway & $^4$Imperial College \\ 
   Switzerland & United Kingdom & United Kingdom & United Kingdom
  \end{tabular}
\\
}

\begin{document}

\maketitle

% -- abstract -- %
\vspace{-2mm}
\begin{abstract}
    %In the standard representation learning on graphs, the relational graph structure is typically encoded in a low-dimensional feature vector that is assigned to each node in the network. However, when these representations are learned with neural nets, they are unfeasible to interpret because their meaning depends on complex nonlinear interactions with uninterpreted features in other layers~\citep{hinton2018deep}. Also, the representations are typically not reproducible in distinct random initial values of the weights. Interpretability (and reproducibility) represent a major caveat of current black-box machine learning models when applied for high-stakes decisions in health care~\citep{rudin2019stop}. 
    Drug repositioning is an attractive cost-efficient strategy for the development of treatments for human diseases. Here, we propose an interpretable model that learns disease self-representations for drug repositioning. Our self-representation model represents each disease as a linear combination of a few other diseases. We enforce proximity in the learnt representations in a way to preserve the geometric structure of the human phenome network --- a domain-specific knowledge that naturally adds relational inductive bias to the disease self-representations. We prove that our method is globally optimal and show results outperforming state-of-the-art drug repositioning approaches. We further show that the disease self-representations are biologically interpretable.

\end{abstract}

% -- intro -- %
\section{Introduction}
% drug repositioning
New drug discovery and development presents several challenges including high attrition rates, long development times, and substantial costs~\citep{Ashburn2004}. Drug repositioning, the process of finding new therapeutic indications for already marketed drugs, has emerged as a promising alternative to new drug development.  It involves the use of de-risked compounds in human, which translates to lower  development costs and shorter development times~\citep{Pushpakom2018}.
%\textit{De novo} drug development presents several challenges including high attrition rates (less than $10$\% overall probability of success~\citep{Ashburn2004}), long development times (total duration of $10$ to $17$ years from the conception of a drug to the marketing stage) and estimated costs in the billions of dollars~\citep{Pushpakom2018}. In light of these challenges, drug repositioning has emerged as a promising alternative. Drug repositioning is the process of finding new indications for already existing drugs. Given that repositioning candidates have been through extensive safety and pharmacological studies, the time, risk and cost associated to the drug development process is significantly reduced~\citep{Ashburn2004}.

% general approaches -- nice paragraph for a journal version
%Early repositioned candidates traditionally emerged as a result of serendipity or empirical observations, but with the surge of high-throughput biological data increasing efforts have been made to develop computational approaches for drug repositioning that can exploit this data in order to surface new candidates with high confidence score. 

% some important works

A wide range of computational approaches has been proposed to predict novel indications for existing drugs. The common assumption underlying these methods is that there is biological or pharmacological relational information between drugs, e.g., chemical structure, and between diseases, e.g., disease phenotypes, that can be exploited for the prediction task. Current methods typically rely on well-defined heuristics and/or hand-crafted features. For instance, the PREDICT~\citep{Gottlieb2011} model %uses well-defined heuristics to 
extracts drug-disease features from multiple similarity measures, and then trains a logistic regression classifier to  predict novel drug indications. Similarly,~\citep{Napolitano2013} obtains features from heterogeneous drug similarity measures and then trains an SVM classifier to predict therapeutic indications of drugs. Other approaches include random walks on bipartite networks~\citep{Luo2016}, and low-rank matrix completion-based approaches~\citep{Luo2018}. 
%Several models have been developed in the past years, almost always employing some form of relational information in the form of biological networks or various similarity measures between drugs and between diseases. PREDICT~\citep{Gottlieb2011} uses well-defined heuristics to extract drug-disease features from multiple similarity graphs, training a logistic regression classifier to weigh their contributions and predict \emph{novel} associations. In~\citep{Napolitano2013} features from heterogeneous drug similarity measures are constructed and employed by an SVM classifier to predict therapeutic indications of drugs. In~\citep{Luo2016}, a Bi-Random Walk algorithm is applied on a heterogeneous network containing diseases and drugs. All these works share the trait of using fixed arbitrary heuristics and/or hand-crafted features.

% recommender systems
%An interesting perspective is advanced in~\citep{Luo2018}, where the authors frame the repositioning problem as a matrix completion task, and define their Drug Repositioning Recommender System (DRRS) method. Given a similarity measure for drugs and a similarity measure for diseases, DRRS constructs a heterogeneous drug–disease interaction network represented by a large drug–disease adjacency matrix whose entries are completed via fast Singular Value Thresholding (SVT) algorithm based on a standard low-rank assumption.

% our model
In this paper, we propose a self-representation learning model for drug repositioning that is able to overcome the limitations of heuristic-based approaches and extend the expressiveness of low-rank factorisation models for matrix completion. Our model builds upon the recent development of high-rank matrix completion based on self-expressive models (SEM)~\citep{elhamifar2016high}, as well as the recent trend of deep learning on graphs~\citep{bronstein2017geometric,monti2017geometric,hamilton2017representation}. 
%
%and \emph{geometrically} %\citep{bronstein2017geometric} extends them to 
We propose a geometric SEM model that integrates relational inductive bias about genetic diseases in the form of a disease phenotype similarity graph. % --- in the inherent self-representation learning process. 
Extensive experiments on a standard benchmark dataset show that our method outperforms existing state-of-the-art approaches in drug repositioning, and that the inclusion of relational inductive bias significantly improves the performance while enhancing model interpretability.

\section{Geometric Self-Expressive Models (GSEM)}

\subsection{Regularisation Framework}

% SEM -> GSEM
The goal of self-expressive models (SEM) is to represent datapoints, i.e., diseases, approximately as a linear combination of a small number of other datapoints. Proposed as a framework for simultaneously clustering and completing high-dimensional data lying over the union of low-dimensional subspaces~\citep{fan2017matrix, wang2018high}, these models can effectively generalise standard low-rank matrix completion models. Here we cast drug repositioning as a high-rank matrix completion task and propose a model that naturally extends SEM by imposing a relational inductive prior between datapoints. Because of the geometrical structure enforced between datapoints, we refer to our model as \emph{Geometric SEM} (GSEM).

% model
%Let us denote our drug-disease association matrix with $n$ drugs and $m$ diseases as $X \in \mathbb{R}^{n\times m}$ (each column is a datapoint) where $X_{ij} = 1$ if drug $i$ is indicated for disease $j$, or $0$ otherwise. GSEM aims at \emph{learning} a  self-representation matrix of coefficients for diseases ($C \in \mathbb{R}^{m\times m}$) such that $\hat{X} \simeq  X C $ where $C$ is sparse according to some sparsity function $\phi(C_{ij})$ and diag($C$) = 0\footnote{Observe that the last constraint is needed to prevent the trivial solution of representing each datapoint with itself ($C = I_m$, the identity matrix)}. GSEM assumes that the datapoints are corrupted by noise; therefore, we propose to solve:

Let us denote our drug disease matrix for $n$ drugs and $m$ diseases with the matrix $X \in \mathbb{R}^{n\times m}$ (each column is a datapoint) where $x_{ij} = 1$ if drug $i$ is associated with disease $j$, and zero otherwise. GSEM aims at \emph{learning} a  sparse zero-diagonal self-representation matrix $C \in \mathbb{R}^{m\times m}$ of coefficients for diseases such that $\hat{X} \approx  X C $, where the null diagonal aims at ruling out the trivial solution $C = I$. 
%

%\begin{equation}\label{eq:GSEM}
%\begin{split}
%&\min_{C}\mathcal{J}(C) = \frac{1}{2} \Vert %X-XC\Vert_F^2 + \sum_{i,j} \phi(C_{ij}) + %\frac{\alpha}{2} \Vert C\Vert_{\mathcal{D,G}}^2 + %\gamma \Tr(C) \\
%&\text{subject to the non-negative constraints}\,\, %C\geq 0.
%\end{split}
%\end{equation}

%where $\Vert .\Vert_F$ is the Frobenius norm, and $ \Vert .\Vert_{\mathcal{D,G}}$ is the Dirichlet semi-norm on the graph $\mathcal{G} = (\mathcal{V,E})$ with $m$ nodes $v_i \in \mathcal{V}$, edges $(v_i, v_j) \in \mathcal{E}$, weighted adjacency matrix $G \in \mathbb{R}^{m\times m}$ and degree matrix $D = diag(\sum_{i} G_{ij})$.%, representing  relational information between diseases.

To this end, we minimise the following cost function:

\begin{equation}\label{eq:GSEM}
\begin{split}
&\min_{C \geq 0}\,\,\, \underbrace{\frac{1}{2} \Vert X-XC\Vert_F^2}_{\text{self-representation}} + \underbrace{
\frac{\beta}{2}\Vert C\Vert_F^2 + \lambda \Vert C\Vert_1}_{\text{sparsity}} 
%\phi(C_{ij}) 
+ \underbrace{\frac{\alpha}{2} \Vert C\Vert_{\mathcal{D},\mathcal{G}}^2}_{\text{smoothness}} + \underbrace{\gamma \Tr(C)}_{\text{null diagonal}} 
%\\
%&\text{subject to the non-negative constraints}\,\, C\geq 0.
\end{split}
\end{equation}
where $\Vert .\Vert_F$ denotes the Frobenius norm, and $ \Vert C\Vert_{\mathcal{D},G}$ % = \sum_{i,j} w_{ij} \| c_i - c_j\|^2 = \Tr(C L_G C^\top)$ 
the Dirichlet norm on the graph $\mathcal{G} = (\{1,\hdots, m\},\mathcal{E},W)$, i.e., the weighted undirected graph with edge weights $w_{ij} > 0$ if $(i,j) \in \mathcal{E}$ and zero otherwise, representing the similarities between diseases. We further denote the cost function in Eq. (\ref{eq:GSEM}) by $\mathcal{Q}(C)$. 
%, representing a relational network information between diseases.
The first term in Eq. (\ref{eq:GSEM}) is the \emph{self-representation constraint}, which aims at learning a matrix of coefficients $C$ such that $XC$ is a good reconstruction of the original matrix $X$. The second term is the \emph{sparsity constraint}, which uses the elastic-net regularisation %$\phi(C_{ij}) = \frac{\beta}{2}\Vert C_{ij}\Vert^2 + \lambda \Vert C_{ij}\Vert_1$ --- which has shown 
known to impose sparsity and grouping-effect %robustness to noise 
%irrelevant features 
\citep{ng2004feature, zou2005regularization}. 
%-- parameterised by constant values $\beta, \lambda > 0$. 
The fourth term is a penalty for diagonal elements to prevent the trivial solution $C = I$ by imposing $\mathrm{diag}(C) = 0$ (together with $C\geq 0$). Typically, $\gamma \gg 0$ is used. Relational inductive bias between diseases is imposed by the third {\em smoothness term} in Eq.\ref{eq:GSEM}~\citep{ma2011recommender,kalofolias2014matrix,monti2017geometric}, incorporating geometric structure into the self-representation matrix $C$ from the disease-disease similarity graph $\mathcal{G}$. Ideally, nearby points in $\mathcal{G}$ should have similar coefficients in $C$, which can be obtained by minimising:
\begin{equation}
 \sum_{i,j} G_{ij} \Vert c_i - c_j\Vert^2 = \Tr(CLC^T) =  \Vert C\Vert_{\mathcal{D,G}}^2
\end{equation}
where $c_i$ and $c_j$ represent column vectors of $C$ and $L = D - W$ is the graph Laplacian. The parameter $\alpha > 0$ in Eq.\ref{eq:GSEM} weighs the importance of the smoothness constraint for the prediction. Finally, to favour interpretability of the learned self-representations, we followed~\citep{lee1999learning} and also imposed a \emph{non-negativity constraint} on $C$.  

%The analysis of our data matrix $X$ reveals that the matrix has a high-rank: 76\% (238/313) of singular values are non-zero. Therefore, we cast the problem of drug repositioning as a high-rank matrix completion problem. The goal of our Geometric Self-Expressive Model (GSEM) is to \emph{learn} a sparse matrix of coefficients for diseases ($C \in \mathbb{R}^{m\times m}$). The data matrix $X$ is then approximated by:

\subsection{The multiplicative learning algorithm}

To minimise Eq.~\ref{eq:GSEM}, we developed an efficient multiplicative learning algorithm with theoretical guarantees of convergence. Our algorithm consists in iteratively applying the following rule:

\begin{equation} \label{eq:algoGSEM}
   c_{ij} \leftarrow  c_{ij}  \frac{(X^\top X +  \alpha CW)_{ij}}{(X^\top XC +  \alpha CD + \beta C + \lambda + \gamma I)_{ij}}    
\end{equation}  
 
We shall prove that the cost function $\mathcal{Q}(C)$ is convex and therefore, our multiplicative rule in Eq. (\ref{eq:algoGSEM}) converges to a {\em global minimum}.

\begin{lemma} \label{lemma:convexity}
The cost function $\mathcal{Q}(C)$ in Eq. \ref{eq:GSEM} is convex in the feasible region $C \geq 0$.
\end{lemma}
\begin{hproof}
We need to prove that the Hessian is a positive semi-definite (PSD) matrix. That is, for a non-zero vector $h \in \mathbb{R}^{m}$ the following condition is met $h^T\nabla^2\mathcal{Q}(C)h \geq 0$. The graph Laplacian is PSD by definition. The remaining terms in the Hessian ($X^\top X + \beta I$) are also PSD. Therefore, $\mathcal{Q}(C)$ is convex in $C \geq 0$. %See supplementary materials for complete proof.
\end{hproof}

\begin{theorem}[Convergence]\label{thm:convergence}
The cost function $\mathcal{Q}(C)$ in Eq. (\ref{eq:GSEM}) converges to a global minimum under the multiplicative update rule in (\ref{eq:algoGSEM}).    
\end{theorem}
\begin{proof} We need to show that our algorithm satisfies the Karush-Khun-Tucker (KKT) complementary conditions, which are both necessary and sufficient conditions for a global solution point given the convexity of the cost function (lemma~\ref{lemma:convexity})~\citep{li2006relationships, boyd2004convex}. KKT requires $c_{ij} \geq 0$ and $(\nabla \mathcal{Q}(C))_{ij}c_{ij} = 0$. The first condition holds with non-negative initialisation of $C$. For the second condition, the gradient is:
$
\nabla \mathcal{Q}(C) = - X^\top X - \alpha CW +  X^\top X C + \alpha CD +\beta C + \lambda + \gamma I
$, and according to the second KKT condition, at convergence $C = C^*$ we have
$
(X^\top X C^* + \alpha C^*D +\beta C^* + \lambda + \gamma I)_{ij} c^*_{ij} - (X^\top X + \alpha C^*W)_{ij} c^*_{ij} = 0
$,   
which is identical to (\ref{eq:algoGSEM}). That is, the multiplicative rule converges to a global minimum solution point.
\end{proof}

% Supplementary
%\paragraph{Algorithm complexity} The most expensive operation in (\ref{eq:algoGSEM}) comes from the denominator term $X^TXC$ for which $\mathcal{O}($\#{\tt iters}$\times m^3)$. The overall complexity can be reduced by pre-computing the constant covariance matrix $X^TX$.

\section{Experimental Results}

\paragraph{Datasets} Drug-disease associations were obtained from the PREDICT dataset~\citep{Gottlieb2011}, a standard benchmark for computational drug repositioning. Our drug-disease matrix $X$ contains $1933$ known drug-disease associations between $593$ drugs (rows) and $313$ genetic diseases (columns). Drugs were extracted from the DrugBank database~\citep{Wishart2018} while diseases were defined using the Online Mendelian Inheritance in Man (OMIM) database~\citep{Hamosh2004}. To build the disease-disease graph $\mathcal{G}$, we used the phenotypic disease similarity constructed by~\citep{van2006text}. Our adjacency matrix $W$ consists of a dense matrix of $313\times 313$ with values ranging in the interval $[0, 1]$.  

\paragraph{Settings} Following previous approaches~\citep{Gottlieb2011,Luo2018}, we frame the computational drug repositioning task as a binary classification problem. We used a ten-fold cross-validation procedure while ensuring that all the drugs had at least one association in the training set. For each fold, we used a validation set for hyperparameter tuning ($\lambda, \beta, \alpha$ and $\gamma$). To filter irrelevant relationships between diseases, we also optimised a threshold parameter $\tau$ such that $w_{ij} = 0$ if $w_{ij} < \tau$. Given that our data matrix $X$ is highly imbalanced ($1.04\%$ density), we measured the performance of the classifier using the area under the precision recall curve (AUPR) --- which has been shown to be more informative than the AUROC in imbalanced scenarios~\citep{saito2015aupr} --- for varying values of ratio between negative ($0$s) and positive ($1$s) labels. We compared the performance of our method against PREDICT~\citep{Gottlieb2011}, a recent low-rank matrix completion model called Drug Repositioning Recommendation System (DRRS)~\citep{Luo2018} and a non-geometric version of our model (with $\alpha = 0$, SEM). For both PREDICT and DRRS, we ran the algorithms with the respective similarity measures used in their original papers. An object-oriented implementation of our model, along with datasets to reproduce our study, is available at \url{https://github.com/Noired/GSEM.git}.

% datasets -- side information
%For both PREDICT and DRRS baselines, the originally employed drug and disease similarity measures were used. As for our proposed GSEM, we experimented with both the two disease affinities used in PREDICT, i.e., \textit{phenotypic} and \textit{human phenotype}, and found the former one to yield best performance. \textit{Phenotypic} similarity from PREDICT is, accordingly, the only one included in our model.

%\subsection{Task and baselines}

% the task
%Our model is evaluated on a drug repositioning task, i.e., the goal is to predict novel indications for drugs known to have \emph{at least one} approved disease indication. In practice, starting from only part of all the considered approved drug-disease associations, we would like to learn to predict the rest of them, held out as a test set.

% baseline
%We compared the results of our proposed GSEM with those of three baselines: PREDICT~\citep{Gottlieb2011}, DRRS~\citep{Luo2018} and SEM, i.e., the standard self-expressive model obtained as specific case of our proposed GSEM with no graph regularization enforced.

%\subsection{Evaluation}
\paragraph{Performance Evaluation}

Figure~\ref{fig:aupr} summarises the performance of the methods. We observed that while GSEM is only marginally better than DRRS and PREDICT in the balanced scenario (negative-to-positive ratio {\tt 1:1}, AUPR of $0.950 \pm 0.010$ for GSEM, $0.947 \pm 0.012$ for DRRS and $0.932 \pm 0.020$ for PREDICT), it greatly outperforms these competitors in imbalanced scenarios (by 1.0-13.16\% at {\tt 10:1}, 4.67-29.15\% at {\tt 30:1}, 6.48-27.60\% at {\tt 50:1} and 10.31-33.12\% at {\tt 100:1}). Although PREDICT integrates information from seven heterogeneous similarity measures, it performs poorly with high data imbalance, where, a simple SEM model, which does not use any relational prior between diseases, intriguingly outperforms PREDICT and DRRS by 1.93-23.05\% at {\tt 50:1} and 6.03-28.84\% at {\tt 100:1}.  This means that our high-rank model of $X$ based on self-representations is able to better exploit the intrinsic relationships between diseases inferred from the data matrix $X$ --- we checked that indeed the matrix $X$ has a high-rank ({\tt rank} = $238/313$). The contribution of the relational information between diseases becomes clear when comparing the performances of the GSEM versus SEM model: on average, GSEM outperforms SEM by 4.85\% AUPR across all the ratios.

%We lastly notice how the performance of SEM decays w.r.t. its geometric counterpart in all the evaluation settings, still demonstrating good robustness in more unbalanced scenarios.

\begin{figure*}[ht!]
  \centering
  \includegraphics[scale = 0.51]{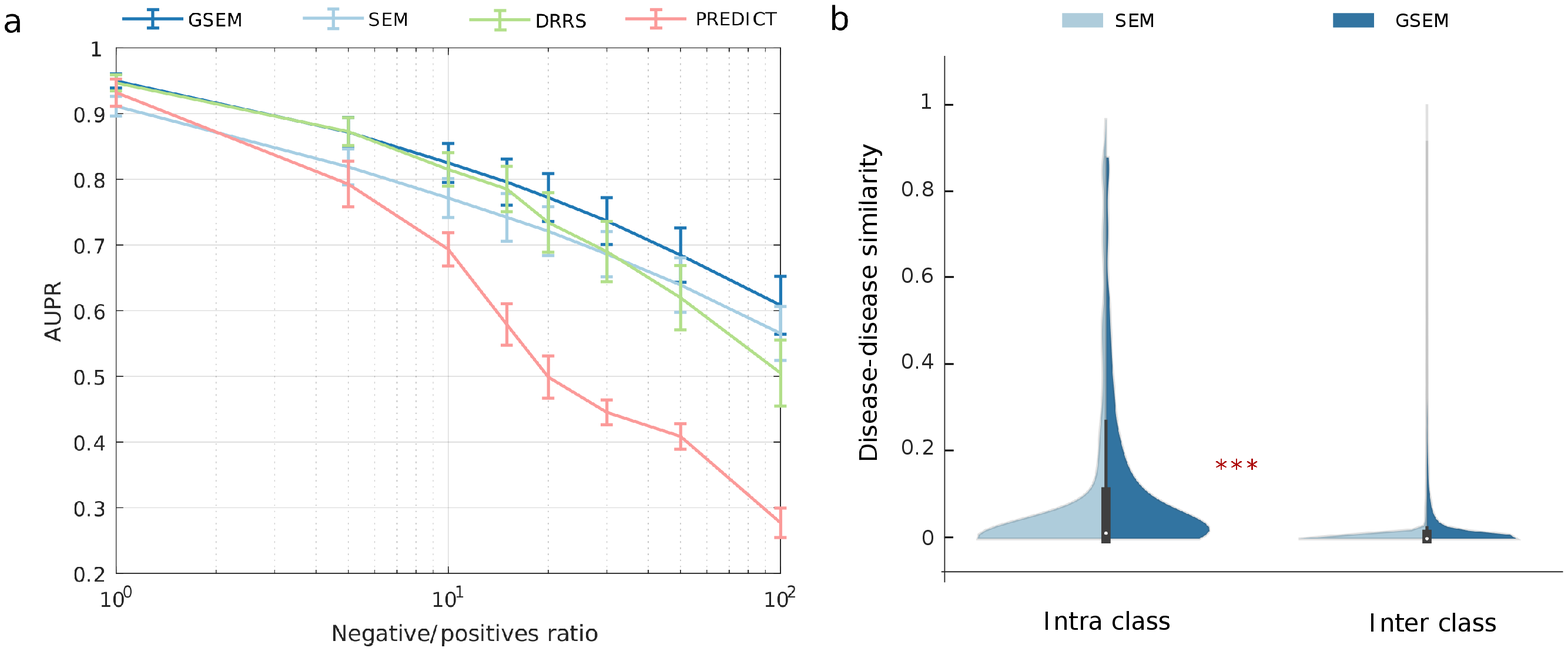}
  \caption{\textbf{(a) Drug repositioning performance}. Mean and standard deviation for AUPR metric computed using ten-fold cross-validation for varying values of negative to positive label ratios (\texttt{ratio} $\in \{1, 5, 10, 15, 20, 30, 40, 50, 100\}$). Performance of our method (GSEM) is reported, along with that of the state-of-the-art DRRS~\citep{Luo2018}, PREDICT~\citep{Gottlieb2011} and standard SEM. Our model performs best in the more realistic unbalanced scenario. \textbf{(b) Model interpretation}. Comparative violin plot of disease self-representation similarities. Diseases were grouped by well-known phenotypic classes. Significance levels between the average intra-class similarities for GSEM and SEM is indicated with asterisks ($p \leq 0.001$, $***$).}
  \label{fig:aupr}
\end{figure*}

% \subsection{Datasets}

% % indications
% For our experiments, obtained from~\citep{Gottlieb2011}, and is collected from multiple data sources. It includes 593 drugs, 313 diseases and 1933 validated drug–disease associations. Drugs are extracted from DrugBank database~\citep{Wishart2018} which is a comprehensive database containing extensive information about drugs and their targets. Diseases are collected from human phenotypes defined in the Online Mendelian Inheritance in Man (OMIM) database~\citep{Hamosh2004} which is a public resource providing information about human genes and diseases.
% \subsection{Experimental settings} Cross-validation procedure ...

% how the train/validation/sets.

% evaluation measure, AUPR, AUC

% density 

% simulate a more realistic scenario 

\section{Model Interpretation}

The effectiveness of our model at predicting new disease indications for drugs prompted us to analyse whether the disease self-representations are informative of the biology underlying drug activity. We assessed model interpretability by exploring the extent to which disease self-representations were related to well-known disease classes. We retrieved disease classes from the International Classification of Diseases of the World Health Organization (WHO) (ICD10CM~\citep{WHO}). We only kept diseases belonging to a unique class, and filtered diseases belonging to a class with less than 5 diseases. For 182 diseases, we obtained 11 distinct classes (see Supplementary Materials). We then defined a disease-disease similarity as the cosine similarity of the rows of $\mathcal{S} = (C+C^T)/2$. For these experiments, we trained our model using all the available data in $X$ and fixed optimal hyper-parameters (see Supplementary Material). Figure~\ref{fig:aupr}(b) shows the violin plots of the distribution of similarities for diseases belonging to the same class (intra-class) versus diseases belonging to distinct classes (inter-class), for both GSEM and SEM. We observed that for both GSEM and SEM, the intra-class similarity is significantly higher than the inter-class similarity (Wilcoxon Sum Rank Significance, $p < 3.62\times 10^{-107}$ and $p < 3.55\times 10^{-249}$ for SEM and GSEM, respectively), meaning that both models learn biologically meaningful self-representations for diseases. However, the intra-class disease self-representation similarity is significantly higher for GSEM than for SEM (Wilcoxon Sum Rank Significance, $p < 1.23\times 10^{-174}$), suggesting that our geometric model efficiently integrates the relational phenotypic information between diseases encoded in the human phenome network. 

%contains disease-disease networks of the SEM and GSEM models with best performance. We can observe how the GSEM model is able to integrate the inductive bias introduced by the Dirichlet regularization, resulting in a better clustering of diseases sharing category. \\ \\
%We further measured within- and between-class disease similarities for both SEM and GSEM (cosine similarities between diseases, with disease vector representations in $(C + C^T)/2 $). We observed a statistically significant difference of between- and within-class similarities between SEM and GSEM (p-value of ranksums test = $0$ and $1.23e^{-174}$ respectively) indicating that similarity distributions of SEM and GSEM differ significantly.

\section{Conclusions and Future Work}
Inherently interpretable models are critical for applications involving high-stakes decisions such as health care~\citep{rudin2019stop}. These are unfeasible to achieve with neural nets, because the learned representations depend on uninterpreted features in other layers~\citep{hinton2018deep}. Here we proposed an inherently interpretable model, GSEM, that learns self-representations for diseases. Our model effectively integrates the relational inductive bias from the human phenome network~\citep{van2006text} --- better results could possibly be achieved using the~\citep{caniza2015network} disease similarity, which has shown to be effective for the prediction of genes associated to genetic diseases~\citep{caceres2019disease}. Ongoing research includes the integration of relational inductive bias for drugs as well. %which implies imposing Dirichlet graph regularisation constraints on the rows of $X$.

\newpage

\section{Supplementary Material}

\subsection{Model Implementation}

\paragraph{PREDICT}
We re-implemented the PREDICT algorithm in Python 3.6. We computed seven drug-disease features using the original data and following the procedure in~\citep{Gottlieb2011}. To train the model, we used the {\tt LogisticRegression} classifier from {\tt sklearn.linear\_model} to obtain the scores for the drug-disease pairs.  

\paragraph{DRRS}
To run Drug Repositioning Recommender System~\citep{Luo2018}, we used the data and code provided in the original publication\footnote{\url{http://bioinformatics.csu.edu.cn/resources/softs/DrugRepositioning/DRRS/index.html}}. DRRS uses two complementary information for the prediction: the 2D Tanimoto chemical similarity for drugs and phenotype similarities for diseases obtained from MimMiner~\citep{van2006text}.

\paragraph{GSEM}
We implemented our proposed algorithm in Python 3.6. As model learning is based on a multiplicative learning rule similar to that of non-negative matrix factorisation (NMF)~\citep{lee1999learning, lee2001algorithms}, we followed the recommended guidelines in~\citep{berry2007algorithms} to implement it: (i) $C$ was initialised with weights sampled from a uniform distribution between $[0, b)$, with $b = 1\times10^{-2}$; (ii) a small value $\varepsilon \simeq 1\times10^{-16}$ was added to the denominator of the learning algorithm to prevent division by zero. The learning rule is iteratively applied until either of the following stopping condition is met: (i) \texttt{maxiter} iterations are completed; (ii) the relative change $\delta^{(t)}$ in the value of $C$ across two subsequent iterations %between $C^{(t+1)}$ and $C^{(t)}$
is smaller than a predefined termination tolerance \texttt{tol}~\citep{kearfott2000stopping}: 

\begin{equation}
    \delta^{(t)} = \max_{ij} \left( \frac{ |c_{ij}^{(t+1)}-c_{ij}^{(t)}|}{\max_{(i,j)} |c_{ij}^{(t)}| + \varepsilon}\right) < \texttt{tol}
\end{equation} 

SEM was naturally obtained as a special case of GSEM by setting $\alpha=0.0$.

\subsection{Experimental Details}

\subsubsection{Hyperparameter Tuning}

GSEM hyperparameters were tuned within the grid $\alpha, \beta, \lambda \in \{0.0, 0.01, 0.1, 1.0, 10.0, 100.0\}$, $\tau \in \{0.0, 0.25, 0.65, 0.75, 0.85, 0.95\}$, with ten-fold cross-validation on the \texttt{2:1} negative-to-positive ratio setting, the one originally considered in~\citep{Gottlieb2011}. Hyperparameters for SEM were chosen likewise.

Hyperparameters optimizing average performance on validation folds were estimated to be $\alpha = 1.0, \beta = 0.1, \lambda = 0.0, \tau = 0.25$ for GSEM and $\beta = 10.0, \lambda = 0.0$ for SEM. As for model interpretation, in order to support posterior analyses we introduced sparsity by taking the largest values of $\lambda$ not compromising model performance. Accordingly, we set $\lambda = 0.01$ for GSEM and $\lambda = 1.0$ for SEM.

Penalty with coefficient $\gamma = 10^{4}$ was imposed on diagonal elements in every setting, and fitting parameters were set as $\texttt{maxiter} = 3\times10^{3}$ and $\texttt{tol} = 1\times10^{-3}$.
% (for which $\Tr(C) \simeq 10^{-XX}$) for all the experiments.

\subsubsection{Training and Evaluation}

DRRS~\citep{Luo2018}, GSEM and SEM are all matrix completion models and, as thus, the whole drug-disease association matrix is taken as input in training. At each step of cross-validation, test and validation folds were hidden by nullifying the associated positives entries. After model fitting, predictions over such fold entries were compared with ground truth values to evaluate performance.

PREDICT~\citep{Gottlieb2011} was learnt as a standard classification model on samples being positive and negative drug-disease associations.

\subsection{Model Interpretation}

\subsubsection{Disease Classes}

We report here below in Table~\ref{table:cat_count}, the number of disease for each of the $11$ distinct classes considered in model interpretation.

\begin{table}[ht!]
  \caption{Disease category counts}
  \label{table:cat_count}
  \centering
  \small
  \begin{tabular}{ll}
    \toprule
    \cmidrule(r){1-2}
    \textbf{Category} & \textbf{Count} \\
    \midrule
    Diseases of the nervous system & $38$ \\
    Endocrine, nutritional and metabolic diseases & $25$ \\
    Neoplasms & $25$ \\
    Diseases of the musculoskeletal system and connective tissue & $16$ \\
    Diseases of the circulatory system & $13$ \\
    Diseases of the blood(-forming) organs and certain disorders involving the immune mech. & $13$ \\
    Diseases of the skin and subcutaneous tissue & $10$ \\
    Diseases of the eye and adnexa & $9$ \\
    Diseases of the genitourinary system & $9$ \\
    Mental, Behavioral and Neurodevelopmental disorders & $7$ \\
    Symptoms, signs and abnormal clinical and laboratory findings, not elsewhere classified & $7$ \\
    Diseases of the digestive system & $5$ \\
    Certain infectious and parasitic diseases & $5$ \\
    \bottomrule
  \end{tabular}
\end{table}

\subsubsection{Disease Similarity Network}

Figure~\ref{fig:networks} depicts the disease network obtained via cosine similarity on the representations \emph{learnt} by GSEM; here we only considered sufficiently represented classes, i.e., we did not report those associated with less than $10$ diseases. From the figure, it emerges how the network presents class-consistent clustering patterns.

\begin{figure*}[!t]
  \centering
  \includegraphics[scale = 0.60]{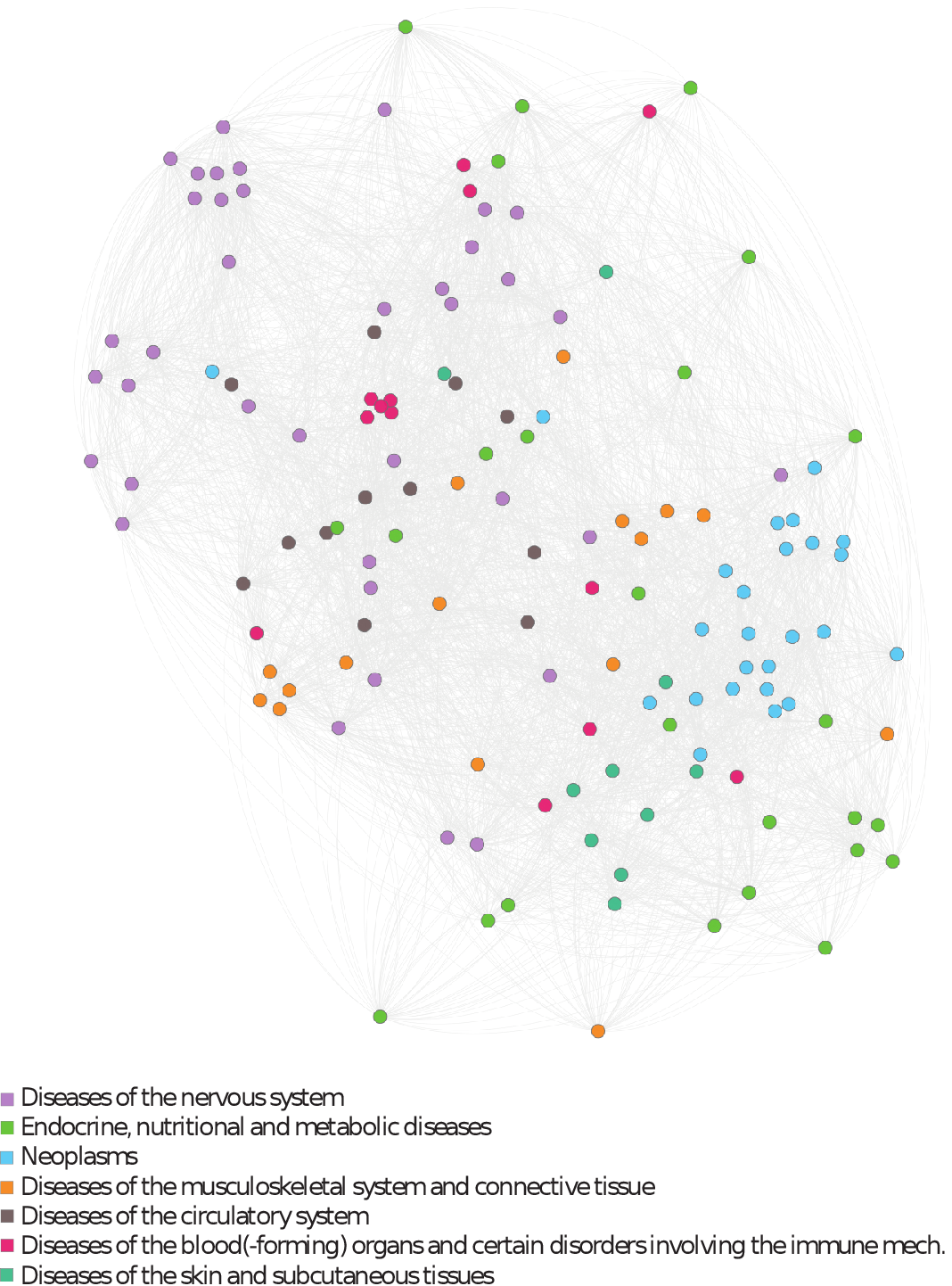}
  \caption{Cosine-similarity network among disease representations learnt by GSEM. Node colour indicates disease class. The plot has been produced via Gephi~\citep{bastian2009gephi} with ForceAtlas embedding.}
  \label{fig:networks}
\end{figure*}
\newpage

\section{Acknowledgements}

We gratefully acknowledge the ERC-Consolidator Grant No. 724228 (LEMAN) for the support to MB and GG, the Vodafone Foundation supporting KV, IL and GG as part of the ongoing DreamLab/DRUGS project and the Imperial NIHR Biomedical Research Center for prospective clinical trials for the support to MB, KV, IL and GG.

AP and DG were supported in part by Biotechnology and Biological Sciences Research Council (BBSRC), grants BB/K004131/1, BB/F00964X/1 and BB/M025047/1 to AP, CONACYT Paraguay Grant INVG01-112 (14-INV-088) and PINV15-315 (14-INV-088), and NSF Advances in Bio Informatics grant 1660648.

FF was supported by the SNF Grant No. 200021E/176315. 

\bibliographystyle{plainnat}
\bibliography{neurips_2019}

\end{document}